\def\year{2019}\relax
\documentclass[letterpaper]{article} 
\usepackage{aaai19}  
\usepackage{times}  
\usepackage{helvet}  
\usepackage{courier}  
\usepackage{url}  
\usepackage{graphicx}  
\usepackage{amssymb}
\usepackage{amsmath}
\usepackage{amsthm}
\usepackage{graphicx}
\usepackage{float}
\newtheorem{theorem}{Theorem}

\newcommand{\tensor}[1]{\boldsymbol{\mathcal{#1}}}

\newcommand{\mat}[1]{\mathbf{#1}}
\newcommand{\vect}[1]{\mathbf{#1}}

\begin{document}
%
\title{Tensor Ring Decomposition with Rank Minimization on Latent Space: An Efficient Approach for Tensor Completion \thanks{This work was supported by JSPS KAKENHI (Grant No. 17K00326, 15H04002, 18K04178), JST CREST (Grant No. JPMJCR1784), and the National Natural Science Foundation of China (Grant No. 61773129). \newline \quad$^\dagger$Corresponding authors: qibin.zhao@riken.jp, cao@sit.ac.jp}}

\author{Longhao Yuan\textsuperscript{1,2},
Chao Li\textsuperscript{2},
Danilo Mandic\textsuperscript{5},
Jianting Cao\textsuperscript{1,2,4,$\dagger$}, Qibin Zhao\textsuperscript{2,3,$\dagger$} \\
\textsuperscript{1}{Graduate School of Engineering, Saitama Institute of Technology, Japan}\\
\textsuperscript{2}{Tensor Learning Unit, RIKEN Center for Advanced Intelligence Project (AIP), Japan}\\
\textsuperscript{3}{School of Automation, Guangdong University of Technology, China}\\
\textsuperscript{4}{School of Computer Science and Technology, Hangzhou Dianzi University, China}\\
\textsuperscript{5}{Department of Electrical and Electronic Engineering, Imperial College London, United Kingdom}\\
}

\maketitle

\begin{abstract}
In tensor completion tasks, the traditional low-rank tensor decomposition models suffer from the laborious model selection problem due to their high model sensitivity. In particular, for tensor ring (TR) decomposition, the number of model possibilities grows exponentially with the tensor order, which makes it rather challenging to find the optimal TR decomposition. In this paper, by exploiting the low-rank structure of the TR latent space, we propose a novel tensor completion method which is robust to model selection. In contrast to imposing the low-rank constraint on the data space, we introduce nuclear norm regularization on the latent TR factors, resulting in the optimization step using singular value decomposition (SVD) being performed at a much smaller scale.  By leveraging the alternating direction method of multipliers (ADMM) scheme, the latent TR factors with optimal rank and the recovered tensor can be obtained simultaneously. Our proposed algorithm is shown to effectively alleviate the burden of TR-rank selection, thereby greatly reducing the computational cost. The extensive experimental results on both synthetic and real-world data demonstrate the superior performance and efficiency of the proposed approach against the state-of-the-art algorithms.
\end{abstract}

\section{Introduction}
Tensor decompositions aim to find latent factors in tensor-valued data (i.e., the generalization of multi-dimensional arrays), thereby casting large-scale and intractable tensor problems into a multilinear tensor latent space of low-dimensionality (very few degrees of freedom designated by the rank). The latent factors within tensor decomposition can be considered as the latent features of data, which makes them an ideal set of bases to predict missing entries when the acquired data is incomplete. The specific forms and operations among latent factors determine the type of tensor decomposition. The most classical and successful tensor decomposition models are the Tucker decomposition (TKD) and the CANDECOMP/PARAFAC (CP) decomposition \cite{kolda2009tensor}. More recently, the matrix product state/tensor-train (MPS/TT) decomposition has become very attractive, owing to its super-compression and computational efficiency properties \cite{oseledets2011tensor}. Currently, a generalization of TT decomposition, termed the tensor ring (TR) decomposition, has been studied across scientific disciplines \cite{zhao2016tensor,zhao2018learning}. These tensor decomposition models have found application in various fields such as machine learning \cite{wang2018wide,novikov2015tensorizing,anandkumar2014tensor,kanagawa2016gaussian}, signal processing \cite{cong2015tensor}, image/video completion \cite{liu2013tensor,zhao2016bayesian}, compressed sensing \cite{gandy2011tensor}, to name but a few. Tensor completion is one of the most important applications of tensor decompositions, with the goal to recover an incomplete tensor from partially observed entries. The theoretical lynchpin in tensor completion problems is the tensor low-rank assumption, and the methods can mainly be categorized into two types: 
(\textrm{i}) tensor-decomposition-based approach and (\textrm{ii}) rank-minimization-based approach.

Tensor decomposition based methods find latent factors of tensor using the incomplete tensor, and then the latent factors are used to predict the missing entries. Many completion algorithms have been proposed based on alternating least squares (ALS) method \cite{grasedyck2015variants,wang2017efficient}, gradient-based method \cite{yuan2017completion,acar2011scalable}, to mention but a few. Though ALS and gradient-based algorithms are free from burdensome hyper-parameter tuning, the performance of these algorithms is rather sensitive to model selection, i.e., rank selection of the tensor decomposition. Moreover, since the optimal rank is generally data-dependent, it is very challenging to specify the optimal rank beforehand. This is especially the case for Tucker, TT, and TR decompositions, for which the rank is defined as a vector; it is therefore impossible to find the optimal ranks by cross-validation due to the immense possibilities. 

Rank minimization based methods employ convex surrogates to minimize the tensor rank. One of the most commonly-used surrogates is the nuclear norm (a.k.a. Schatten norm, or trace norm), which is defined as the sum of singular values of a matrix and it is the most popular convex surrogate for rank regularization. Based on different definitions of tensor rank, various nuclear norm regularized algorithms have been proposed \cite{liu2013tensor,imaizumi2017tensor,liu2014generalized,liu2015trace}. Rank minimization based methods do not need to specify the rank of the employed tensor decompositions beforehand, and the rank of the recovered tensor will be automatically learned from the limited observations. However, these algorithms face multiple large-scale singular value decomposition (SVD) operations on the 2D unfoldings of the tensor when employing the nuclear norm and numerous hyper-parameter tuning, which in turn leads to high computational cost and low efficiency. 

To address the problems of high sensitivity to rank selection and low computational efficiency which are inherent in traditional tensor completion methods, in this paper, we propose a new algorithm named tensor ring low-rank factors (TRLRF) which effectively alleviates the burden of rank selection and reduces the computational cost. By virtue of employing both nuclear norm regularization and tensor decomposition, our model provides performance stability and high computational efficiency. The proposed TRLRF is efficiently solved by the ADMM algorithm and it simultaneously achieves both the underlying tensor decomposition and completion based on TR decomposition. Our main contributions in this paper are:
\begin{itemize}
\item A theoretical relationship between the multilinear tensor rank and the rank of TR factors is established, which allows the low-rank constraint to be performed implicitly on TR latent space. This has led to fast SVD calculation on small size factors. 

\item The nuclear norm is further imposed to regularize the TR-ranks, which enables our algorithm to always obtain a stable solution, even if the TR-rank is inappropriately given. This highlights rank-robustness of the proposed TRLRF algorithm. 

\item An efficient algorithm based on ADMM is developed to optimize the proposed model, so as to obtain the TR-factors and the recovered tensor simultaneously.
\end{itemize}


\section{Preliminaries and Related Works}
\subsection{Notations}
The notations in \cite{kolda2009tensor} are adopted in this paper. A scalar is denoted by a standard lowercase letter or a uppercase letter, e.g., $x, X \in\mathbb{R}$, and a vector is denoted by a boldface lowercase letter, e.g., $\vect{x}\in\mathbb{R}^{I}$. A matrix is denoted by a boldface capital letter, e.g., $\mat{X}\in\mathbb{R}^{I\times J}$. A tensor of order $N\geq 3$ is denoted by calligraphic letters, e.g., $\tensor{X}\in\mathbb{R}^{I_1\times I_2\times\cdots \times I_N}$. The set $\{ \tensor{X}^{(n)}\}_{n=1}^N:=\{ \tensor{X}^{(1)},\tensor{X}^{(2)},\ldots,\tensor{X}^{(N)}\} $ denotes a tensor sequence, with $\tensor{X}^{(n)}$ being the $n$-th tensor of the sequence. Where appropriate, a tensor sequence can also be written as $[\tensor{X}]$. The representations of matrix sequences and vector sequences are designated in the same way. An element of  a tensor $\tensor{X}  \in\mathbb{R}^{I_1\times I_2\times\cdots \times I_N}$ of index $(i_{1},i_{2},\ldots,i_{N})$ is denoted by $\tensor{X}(i_{1},i_{2},\ldots, i_{N})$ or $x_{i_{1}i_{2}\ldots i_{N}}$. The inner product of two tensors $\tensor{X}$, $\tensor{Y}$ with the same size $\mathbb{R}^{I_1\times I_2\times\cdots \times I_N}$ is defined as $\langle \tensor{X},\tensor{Y} \rangle=\sum_{i_1}\sum_{i_2}\cdots\sum_{i_N}x_{i_1 i_2\ldots i_N}y_{i_1 i_2\ldots i_N}$. Furthermore, the Frobenius norm of $\tensor{X}$ is defined by $\left \| \tensor{X} \right \|_F=\sqrt{\langle \tensor{X},\tensor{X} \rangle}$.

We employ two types of tensor unfolding (matricization) operations in this paper. The standard mode-$n$ unfolding \cite{kolda2009tensor} of tensor $\tensor{X}  \in\mathbb{R}^{I_1\times I_2\times\cdots \times I_N}$ is denoted by $\mat{X}_{(n)}\in\mathbb{R}^{I_n \times  {I_1 \cdots I_{n-1} I_{n+1} \cdots I_N}}$. Another mode-$n$ unfolding of tensor $\tensor{X}$ which is often used in TR operations \cite{zhao2016tensor} is denoted by $\mat{X}_{<n>}\in\mathbb{R}^{I_n \times  {I_{n+1} \cdots I_{N} I_{1} \cdots I_{n-1}}}$. Furthermore, the inverse operation of unfolding is matrix folding (tensorization), which transforms matrices to higher-order tensors. In this paper, we only define the folding operation for the first type of mode-$n$ unfolding as $\text{fold}_n(\cdot)$, i.e., for a tensor $\tensor{X}$, we have $\text{fold}_n(\mat{X}_{(n)})=\tensor{X}$.

\subsection{Tensor ring decomposition}

The tensor ring (TR) decomposition is a more general decomposition model than the tensor-train (TT) decomposition. It represents a tensor of higher-order by circular multilinear products over a sequence of low-order latent core tensors, i.e., TR factors. For $n=1,\ldots,N$, the TR factors are denoted by $\tensor{G}^{(n)} \in\mathbb{R}^{R_{n} \times I_{n} \times R_{n+1}}$ and each consists of two rank-modes (i.e mode-$1$ and mode-$3$) and one dimension-mode (i.e., mode-$2$). The syntax $\{R_1, R_2, \ldots, R_N\}$ denotes the TR-rank which controls the model complexity of TR decomposition. The TR decomposition applies trace operations and all of the TR factors are set to be 3-order; thus the TR decomposition relaxes the rank constraint on the first and last core of TT to $R_1=R_{N+1}$. Moreover, TR decomposition linearly scales to the order of the tensor, and in this way it overcomes the `curse of dimensionality'. In this case, TR can be considered as a linear combination of TTs and hence offers a powerful and generalized representation ability. The element-wise relation of TR decomposition and the generated tensor is given by:
\begin{equation}
\label{tr_relation1}
\tensor{X}(i_1,i_2,\ldots,i_N)=\text{Trace}\left \{  \prod_{n=1}^N \mat{G}^{(n)}_{i_n} \right \},
\end{equation}
where $\text{Trace}\{ \cdot \}$ is the matrix trace operation, $ \mat{G}^{(n)}_{i_n} \in\mathbb{R}^{R_n\times R_{n+1}}$ is the $i_n$-th mode-$2$ slice matrix of $\tensor{G}^{(n)}$, which can also be denoted by $\tensor{G}^{(n)}(:,i_n,:)$ according to the Matlab notation.

\subsection{Tensor completion}

\subsubsection{Completion by TR decomposition}
Tensor decomposition based algorithms do not directly employ the rank constraint to the object tensor. Instead, they try to find the low-rank representation (i.e., tensor decompositions) of the incomplete data from the observed entries. The obtained latent factors of the tensor decomposition are used to predict the missing entries. For model formulation, the tensor completion problem is set as a weighted least squares (WLS) model. Based on different tensor decompositions, various tensor completion algorithms have been proposed, e.g., weighted CP \cite{acar2011scalable}, weighted Tucker \cite{filipovic2015tucker}, TRWOPT \cite{yuan2018higher} and TRALS \cite{wang2017efficient}. To the best of our knowledge, there are two proposed TR-based tensor completion algorithms: the TRALS and TRWOPT. They apply the same optimization model which is formulated as:
\begin{equation}
\min\limits_{[\tensor{G}]} \Vert P_{\Omega}(\tensor{T}-\Psi([\tensor{G}])) \Vert_F^2,
\label{tdc}
\end{equation}
where the optimization objective is the TR factors, $[\tensor{G}]$, $P_{\Omega}(\tensor{T})$ denotes all the observed entries w.r.t. the set of indices of observed entries represented by $\Omega$, and $\Psi([\tensor{G}])$ denotes the approximated tensor generated by $[\tensor{G}]$. Every element of $\Psi([\tensor{G}])$ is calculated by equation \eqref{tr_relation1}. The two algorithms are both based on the model in (\ref{tdc}). However, TRALS applies alternative least squares (ALS) method and TRWOPT uses a gradient-based algorithm to solve the model, respectively. They perform well for both low-order and high-order tensors due to the high representation ability and flexibility of TR decomposition. However, these algorithms are shown to suffer from high sensitiveness to rank selection, which would lead to high computational cost.

\subsubsection{Completion by nuclear norm regularization}
The model of rank minimization-based tensor completion can be formulated as:
\begin{equation}
\label{bmc}
\min \limits_{\tensor{X}} \ \  \text{Rank}(\tensor{X})+\frac{\lambda}{2}\Vert P_{\Omega}(\tensor{T}-\tensor{X})\Vert_F^2,
\end{equation}
where $\tensor{X}$ is the recovered low-rank tensor, and $\text{Rank}(\cdot)$ is a rank regularizer. The model can therefore find the low-rank structure of the data and approximate the recovered tensor. Because determining the tensor rank is an NP-hard problem \cite{hillar2013most,kolda2009tensor}, work in \cite{liu2013tensor} and \cite{signoretto2014learning} extends the concept of low-rank matrix completion and defines tensor rank as a sum of the rank of mode-$n$ unfolding of the object tensor. Moreover, the convex surrogate named nuclear norm is applied to the tensor low-rank model and it simultaneously regularizes all the mode-$n$ unfoldings of the object tensor. In this way, the model in \eqref{bmc} can be reformulated as:
\begin{equation}
\label{overlapped}
\min_{\tensor{X}}\sum_{n=1}^N\Vert \mat{X}_{(n)} \Vert_*+\frac{\lambda}{2}\Vert P_{\Omega}(\tensor{T}-\tensor{X})\Vert_F^2, 
\end{equation}
where $\Vert \cdot \Vert_*$ denotes the nuclear norm regularization in the form of a sum of the singular values of the matrix. Usually, the model is solved by ADMM algorithms and it is shown to have fast convergence and good performance when data size is small. However, when dealing with large-scale data, the multiple SVD operations in the optimization step will be intractable due to high computational cost.

\section{Tensor Ring Low-rank Factors}
To solve the issues traditional tensor completion methods have, we impose low-rankness on each of the TR factors and so that our basic tensor completion model is formulated as follow:
\begin{equation}
\label{ori_mod}
\begin{aligned}
\min \limits_{[\tensor{G}],\tensor{X}} \ \ &\sum_{n=1}^N \Vert \tensor{G}^{(n)}\Vert_*+ \frac{\lambda}{2}\Vert  \tensor{X}-\Psi([\tensor{G}])\Vert_F^2,\\&
s.t.\ P_\Omega(\tensor{X})=P_\Omega(\tensor{T}).
\end{aligned}
\end{equation}

To solve (\ref{ori_mod}), we first need to deduce the relation of the tensor rank and the corresponding core tensor rank, which can be explained by the following theorem.
\begin{theorem}
Given an $N$-th order tensor $\tensor{X}\in\mathbb{R}^{I_1\times I_2\times\cdots \times I_N}$ which can be represented by  equation (\ref{tr_relation1}), then the following inequality holds for all $n=1,\ldots,N$: 
\begin{equation}
\text{Rank}(\mat{G}_{(2)}^{(n)}) \geq \text{Rank}(\mat{X}_{(n)}).
\end{equation}
\end{theorem}
\begin{proof}
For the $n$-th core tensor $\tensor{G}^{(n)}$, according to the work in \cite{zhao2016tensor}, we have:
\begin{equation}
\mat{X}_{<n>}=\mat{G}_{(2)}^{(n)}(\mat{G}_{<2>}^{(\neq n)})^T,
\end{equation}
where $\tensor{G}^{(\neq n)}\in\mathbb{R}^{R_{n+1}\times \prod_{i=1, i\neq n}^N I_i \times R_n}$ is a subchain tensor generated by merging all but the $n$-th core tensor. Hence, the relation of the rank satisfies:
\begin{equation}
\begin{aligned}
 \text{Rank}(\mat{X}_{<n>})& \leq \text{min}\{\text{Rank}(\mat{G}^{(n)}_{(2)}) ,\text{Rank}(\mat{G}_{<n>}^{(\neq n)})\}\\& \leq \text{Rank}(\mat{G}^{(n)}_{(2)}).
 \end{aligned}
 \end{equation}
 The proof is completed by
 \begin{equation}
\begin{aligned}
\text{Rank}(\mat{X}_{<n>})=\text{Rank}(\mat{X}_{(n)}) \leq \text{Rank}(\mat{G}^{(n)}_{(2)}).
 \end{aligned}
 \end{equation}
 \end{proof}
This theorem proves the relation between the tensor rank and the rank of the TR factors. The rank of mode-$n$ unfolding of the tensor $\tensor{X}$ is upper bounded by the rank of the dimension-mode unfolding of the corresponding core tensor $\tensor{G}^{(n)}$, which allows us to impose a low-rank constraint on $\tensor{G}^{(n)}$. By the new surrogate, our model (\ref{ori_mod}) is reformulated by:
\begin{equation}
\label{mod}
\begin{aligned}
\min \limits_{[\tensor{G}],\tensor{X}}  \ &\sum_{n=1}^N \Vert \mat{G}^{(n)}_{(2)} \Vert_*+ \frac{\lambda}{2}\Vert \tensor{X}-\Psi([\tensor{G}])\Vert_F^2\\&
s.t.\ P_\Omega(\tensor{X})=P_\Omega(\tensor{T}).
\end{aligned}
\end{equation}

The above model imposes nuclear norm regularization on the dimension-mode unfoldings of the TR factors, which can largely decrease the computational complexity compared to the algorithms which are based on model (\ref{overlapped}). Moreover, we consider to give low-rank constraints on the two rank-modes of the TR factors, i.e., the unfoldings of the TR factors along mode-$1$ and mode-$3$, which can be expressed by $\sum_{n=1}^N\Vert \mat{G}^{(n)}_{(1)}\Vert_*$+$\sum_{n=1}^N\Vert \mat{G}^{(n)}_{(3)}\Vert_*$. When the model is optimized, nuclear norms of the rank-mode unfoldings and the fitting error of the approximated tensor are minimized simultaneously, resulting in the initial TR-rank becoming the upper bound of the real TR-rank of the tensor, thus equipping our model with robustness to rank selection. The tensor ring low-rank factors (TRLRF) model can be finally expressed as:
\begin{equation}
\label{mod_o}
\begin{aligned}
\min \limits_{[\tensor{G}],\tensor{X}}  \ &\sum_{n=1}^N\sum_{i=1}^3 \Vert \mat{G}^{(n)}_{(i)} \Vert_*+ \frac{\lambda}{2}\Vert \tensor{X}-\Psi([\tensor{G}])\Vert_F^2\\&
s.t.\ P_\Omega(\tensor{X})=P_\Omega(\tensor{T}).
\end{aligned}
\end{equation}
Our TRLRF model has two distinctive advantages. Firstly, the low-rank assumption is placed on tensor factors instead of on the original tensor, this greatly reduces the computational complexity of the SVD operation. Secondly, low-rankness of tensor factors can enhance the robustness to rank selection, which can alleviate the burden of searching for optimal TR-rank and reduce the computational cost in the implementation.

\subsection{Solving scheme}
To solve the model in (\ref{mod_o}), we apply the alternating direction method of multipliers (ADMM) which is efficient and widely used \cite{boyd2011distributed}. Moreover, because the variables of TRLRF model are inter-dependent, we impose auxiliary variables to simplify the optimization. Thus, the TRLRF model can be rewritten as
\begin{equation}
\label{L_mod}
\begin{aligned}
&\min \limits_{[\tensor{M}], [\tensor{G}],\tensor{X}} \ \sum_{n=1}^N\sum_{i=1}^3 \Vert \mat{M}^{(n,i)}_{(i)} \Vert_*+ \frac{\lambda}{2}\Vert \tensor{X}- \Psi([\tensor{G}])\Vert_F^2 ,\\& \quad s.t. \  \mat{M}^{(n,i)}_{(i)}=\mat{G}^{(n)}_{(i)}, n=1,\ldots,N,\ i=1,2,3, \\&\quad\quad\;\; P_\Omega(\tensor{X})=P_\Omega(\tensor{T}),
\end{aligned}
\end{equation}
where $[\tensor{M}]:=\{\tensor{M}^{(n,i)}\}_{n=1,i=1}^{N,3}$ are the auxiliary variables of $[\tensor{G}]$. By merging the equal constraints of the auxiliary variables into the Lagrangian equation, the augmented Lagrangian function of TRLRF model becomes
\begin{equation}
\label{o_Lfunc}
\begin{aligned}
&L	\left( [\tensor{G}],\tensor{X},[\tensor{M}], [\tensor{Y}]\right)\\&=\sum_{n=1}^N\sum_{i=1}^3 \big(\Vert \mat{M}^{(n,i)}_{(i)} \Vert_*+<\tensor{Y}^{(n,i)}, \tensor{M}^{(n,i)}-\tensor{G}^{(n)}>\\  &+\frac{\mu}{2}\Vert  \tensor{M}^{(n,i)}-\tensor{G}^{(n)}\Vert_F^2 \big) +\frac{\lambda}{2}\Vert  \tensor{X}- \Psi([\tensor{G}])\Vert_F^2 ,\\
&\quad s.t.\ P_\Omega(\tensor{X})=P_\Omega(\tensor{T}),
\end{aligned}
\end{equation}
where $[\tensor{Y}]:=\{\tensor{Y}^{(n,i)}\}_{n=1,i=1}^{N,3}$ are the Lagrangian multipliers, and $\mu>0$ is a penalty parameter. For $n=1,\ldots,N$, $i=1,2,3$, $\tensor{G}^{(n)}$, $\tensor{M}^{(n,i)}$ and $ \tensor{Y}^{(n,i)}$ are each independent, so we can update them by the updating scheme below.

\subsubsection{Update of $\tensor{G}^{(n)}$.}
By using \eqref{o_Lfunc}, the augmented Lagrangian function w.r.t. ${\tensor{G}^{(n)}}$ can be simplified as
\begin{equation}
\label{LG_o}
\begin{aligned}
L(\tensor{G}^{(n)})&=\sum_{i=1}^3\frac{\mu}{2}\Big\Vert  \tensor{M}^{(n,i)}-\tensor{G}^{(n)}+\frac{1}{\mu}\tensor{Y}^{(n,i)}  \Big\Vert^2_F\\&+\frac{\lambda}{2}\big\Vert \tensor{X}-\Psi([\tensor{G}]) \big\Vert^2_F+C_{\tensor{G}},
 \end{aligned}
\end{equation}
where the constant $C_{\tensor{G}}$ consists of other parts of the Lagrangian function which is irrelevant to updating $\tensor{G}^{(n)}$. This is a least squares problem, so for $n=1, \ldots,N$, $\tensor{G}^{(n)}$ can be updated by
\begin{equation}
\label{LG_u}
\begin{aligned}
&\tensor{G}^{(n)}=\text{fold}_{2}\Big(\big(\sum_{i=1}^{3}(\mu \mat{M}_{(2)}^{(n,i)} +\mat{Y}_{(2)}^{(n,i)})\\&+\lambda \mat{X}_{<n>}\mat{G}^{(\neq n)}_{<2>} \big)\big(\lambda\mat{G}^{(\neq n),T}_{<2>}\mat{G}^{(\neq n)}_{<2>}+3\mu \mat{I}\big)^{-1}\Big),
 \end{aligned}
\end{equation}
where $\mat{I}\in\mathbb{R}^{R_n^2\times R_n^2}$ denotes the identity matrix.

\subsubsection{Update of $\tensor{M}^{(n,i)}$.}
For $i=1,2,3$, the augmented Lagrangian functions w.r.t. $[\tensor{M}]$ is expressed as
\begin{equation}
\label{LM_o}
\begin{aligned}
L(\tensor{M}^{(n,i)})&=\frac{\mu}{2}\big\Vert \tensor{M}^{(n,i)}-\tensor{G}^{(n)}+\frac{1}{\mu}\tensor{Y}^{(n,i)} \big\Vert_F^2\\&+ \big\Vert \mat{M}^{(n,i)}_{(i)}\big \Vert_*+C_{\tensor{M}}.
\end{aligned}
\end{equation}
The above formulation has a closed-form \cite{cai2010singular}, which is given by
\begin{equation}
\label{LM_u}
\begin{aligned}
\tensor{M}^{(n,i)}=\text{fold}_i\Big(D_{\frac{1}{\mu}}\big({\mat{G}^{(n)}_{(i)}}-\frac{1}{\mu}\mat{Y}^{(n,i)}_{(i)}\big)\Big),
\end{aligned}
\end{equation}
where $D_{\beta}(\cdot)$ is the singular value thresholding (SVT) operation, e.g., if $\mat{U}\mat{S}\mat{V}^T$ is the singular value decomposition of matrix $\mat{A}$, then $D_\beta(\mat{A})=\mat{U}max\{\mat{S}-\beta \mat{I},0\}\mat{V}^T$.

\subsubsection{Update of $\tensor{X}$.}
The augmented Lagrangian functions w.r.t. $\tensor{X}$ is given by
\begin{equation}
\label{LX_o}
\begin{aligned}
L&(\tensor{X})=\frac{\lambda}{2}\big\Vert \tensor{X}-\Psi([\tensor{G}]) \big\Vert^2_F+C_{\tensor{X}}, \\&s.t. \ P_\Omega(\tensor{X})=P_\Omega(\tensor{T}),
\end{aligned}
\end{equation}
which is equivalent to the tensor decomposition based model in (\ref{tdc}). The expression for $\tensor{X}$ is updated by inputing the observed values in the corresponding entries, and by approximating the missing entries by updated TR factors $[\tensor{G}]$ for every iteration, i.e.,
\begin{equation}
\label{LX_u}
\begin{aligned}
\tensor{X}=P_{\Omega}(\tensor{T})+P_{\bar{\Omega}}(\Psi([\tensor{G}])),
\end{aligned}
\end{equation}
where $\bar{\Omega}$ is the set of indices of missing entries which is a complement to $\Omega$.

\subsubsection{Update of $\tensor{Y}^{(n,i)}$.}
For $n=1,\ldots,N$ and $i=1,2,3$, the Lagrangian multiplier $ \tensor{Y}^{(n,i)}$ is updated as
\begin{equation}
\label{LY_u}
\begin{aligned}
 \tensor{Y}^{(n,i)}=\tensor{Y}^{(n,i)}+\mu\big(\tensor{M}^{(n,i)}-\tensor{G}^{(n)}\big).
\end{aligned}
\end{equation}
In addition, the penalty term of  the Lagrangian functions $L$ is restricted by $\mu$ which is also updated for every iteration by $\mu=max\{\rho\mu,\mu_{max}\}$, where $1<\rho<1.5$ is a tuning hyper parameter.

The ADMM based solving scheme is updated iteratively based on the above equations. Moreover, we consider to set two optimization stopping conditions: (\textrm{i}) maximum number of iterations $k_{max}$ and (\textrm{ii}) the difference between two iterations (i.e., $\Vert \tensor{X}-\tensor{X}_{last} \Vert_F / \Vert  \tensor{X}\Vert_F$) which is thresholded by the tolerance $tol$. The implementation process and hyper-parameter selection of TRLRF is summarized in Algorithm 1. It should be noted that our TRLRF model is non-convex, so the convergence to the global minimum cannot be theoretically guaranteed. However, the convergence of our algorithm can be verified empirically (see experiment details in Figure \ref{conver}). Moreover, the extensive experimental results in the next section also illustrate the stability and effectiveness of TRLRF. 
\begin{figure}[htb]
\begin{center}
\includegraphics[width=1\linewidth]{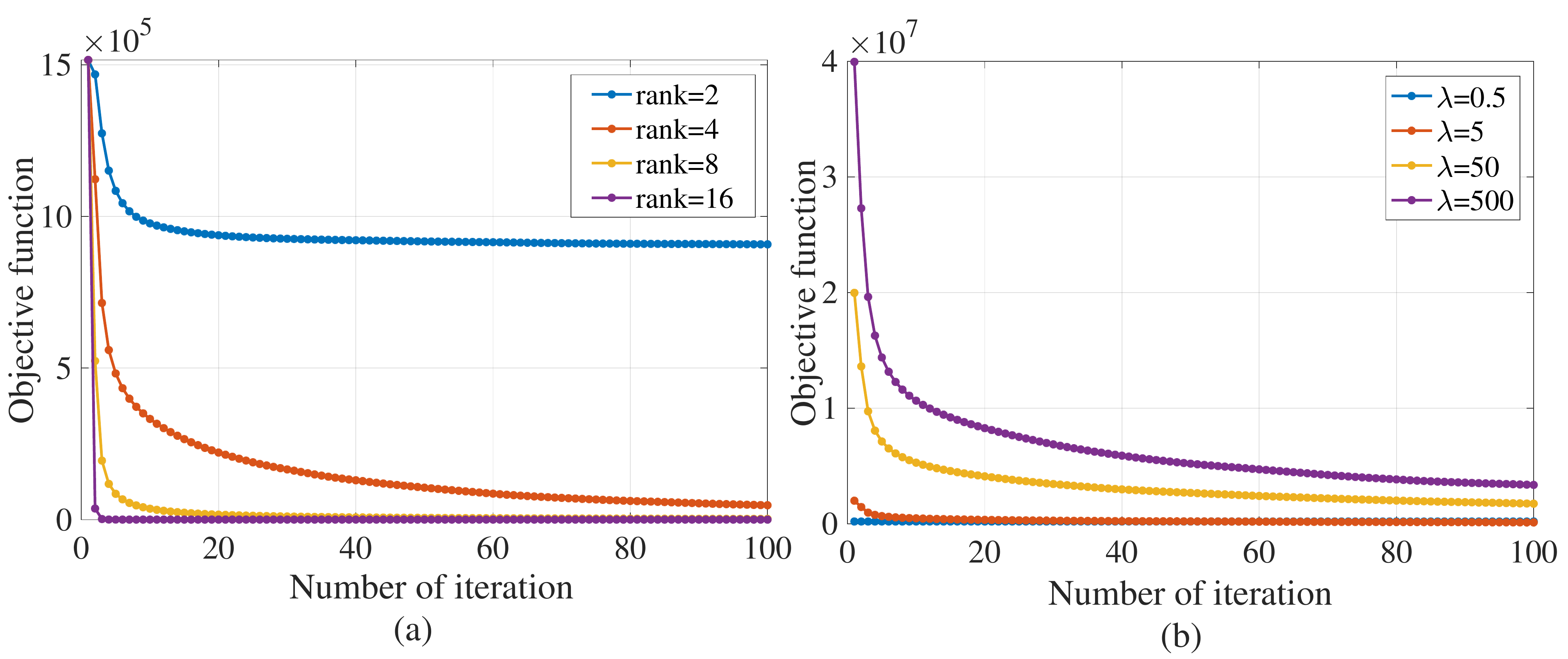}
\end{center}
\caption{Illustration of convergence property for TRLRF under different hyper-parameter choices. A synthetic tensor with TR structure (size $7\times8\times 7\times8$ with TR-rank \{4,4,4,4\}, missing rate 0.5) is tested. The experiment records the change of the objective function values along the number of iterations. Each independent experiment is conducted 100 times and the average results are shown in the graphs. Panels (a) and (b) represent the convergence curve when TR-rank and $\lambda$ are changed respectively.}
\label{conver}
\end{figure}
\begin{table}[h]
\centering
\begin{center}
\begin{tabular}{l}
\hline
\textbf{Algorithm 1.}  Tensor ring low-rank factors (TRLRF) \\
\hline
1: \textbf{Input}: $P_{\Omega}(\tensor{T})$, initial TR-rank $\{R_n\}_{n=1}^N$ . \\
2: \textbf{Initialization}: For $n=1,\ldots,N$, $i=1,2,3$, \\\quad random sample $\tensor{G}^{(n)}$ by distribution $N\sim(0,1)$, \\\quad$\tensor{Y}^{(n,i)}=0$, $\tensor{M}^{(n,i)}=0$, $\lambda=5$, $\mu^0=1$, $\mu_{max}$ \\\quad$=10^2, \rho=1.01$, $tol=10^{-6}$, $k=0$, $k_{max}=300$.  \\
3:\; \textbf{For} $k=1$ to $k_{max}$ \textbf{do}\\
4:\;\;\;\; $\tensor{X}_{last}=\tensor{X}$. \\
5:\;\;\;\; Update $\{\tensor{G}^{(n)}\}_{n=1}^N$ by (\ref{LG_u}). \\
6:\;\;\;\; Update $\{\tensor{M}^{(n,i)}\}_{n=1,i=1}^{N,3}$ by (\ref{LM_u}). \\
7:\;\;\;\; Update $\tensor{X}$ by (\ref{LX_u}). \\
8:\;\;\;\; Update $\{ \tensor{Y}^{(n,i)}\}_{n=1,i=1}^{N,3}$ by (\ref{LY_u}). \\
9:\;\;\;\; $\mu=max(\rho\mu,\mu_{max})$ \\
6:\;\;\;\; \textbf{If} $\Vert \tensor{X}-\tensor{X}_{last} \Vert_F / \Vert  \tensor{X}\Vert_F<tol$, \textbf{break} \\
7: \;\textbf{End for}  \\
8: \;\textbf{Output}: completed tensor $\tensor{X}$ and TR factors $[\tensor{G}]$. \\
\hline
\end{tabular}
\end{center}
\end{table}

\subsection{Computational complexity}

We analyze the computational complexity of our TRLRF algorithm as follows. For a tensor $\tensor{X}\in\mathbb{R}^{I_1\times I_2\times \cdots\times I_N}$, the TR-rank is set as $R_1=R_2=\cdots=R_N=R$, then the computational complexity of updating $[\tensor{M}]$ represents mainly the cost of SVD operation, which is $\tensor{O}(\sum_{n=1}^N2I_nR^3+I_n^2R^2)$. The computational complexities incurred calculating $\mat{G}_{<2>}^{(\neq n)}$ and updating $[\tensor{G}]$ are $\tensor{O}(NR^3\prod_{i=1, i\neq n}^N I_i)$ and $\tensor{O}(NR^2\prod_{i=1}^N I_i+NR^6)$, respectively. If we assume $I_1=I_2=\cdots=I_N=I$, then overall complexity of our proposed algorithm can be written as $\tensor{O}(NR^2I^N+NR^6)$.

Compared to HaLRTC and TRALS which are the representative of the nuclear-norm-based and the tensor decomposition based algorithms, the computational complexity of HaLRTC is $\tensor{O}(NI^{N+1})$. Since TRALS is based on ALS method and TR decomposition, its computational complexity is $\tensor{O}(PNR^4I^N+NR^6)$, where $P$ denotes the observation rate. We can see that the computational complexity of our TRLRF is similar to that of the two related algorithms. However, the desirable characteristic of rank selection robustness of our algorithm can help relieve the workload for model selection in practice, and thus the computational cost can be reduced. Moreover, though the computational complexity of TRLRF is of high power in $R$, due to the high representation ability and flexibility of TR decomposition, the TR-rank is always set as a small value. In addition, from experiments, we find out that our algorithm is capable of working efficiently for high-order tensors so that we can tensorize the data to a higher-order tensor and choose a small TR-rank to reduce the computational complexity. 

\section{Experimental Results}

\subsection{Synthetic data}
We first conducted experiments to testify the rank robustness of our algorithm by comparing TRALS, TRWOPT, and our TRLRF. To verify the performance of the three algorithms, we tested two tensors of size $20\times20\times20\times20$ and $7 \times8\times7\times8\times7\times8$. The tensors were generated by TR factors of TR-ranks $\{6,6,6,6\}$ and $\{4,4,4,4,4,4\}$ respectively. The values of the TR factors were drawn from i.i.d. Gaussian distribution $N \sim (0,0.5)$. The observed entries of the tensors were randomly removed by a missing rate of $0.5$, where the missing rate is calculated by $1-M/\text{num}(\tensor{T}_{real})$ and $M$ is the number of sampled entries (i.e., observed entries). We recorded the completion performance of the three algorithms by selecting different TR-ranks. The evaluation index was RSE which is defined by $\text{RSE}=\Vert  \tensor{T}_{real}-\tensor{X} \Vert_F / \Vert \tensor{T}_{real}\Vert_F$, where $\tensor{T}_{real}$ is the known tensor with full observations and $\tensor{X}$ is the recovered tensor calculated by each tensor completion algorithm. The hyper-parameters of our TRLRF were set according to Algorithm 1. All the hyper-parameters of TRALS and TRWOPT are set according to the recommended settings in the corresponding papers to get the best results. 

\begin{figure}[htb]
\begin{center}
\includegraphics[width=0.9\linewidth]{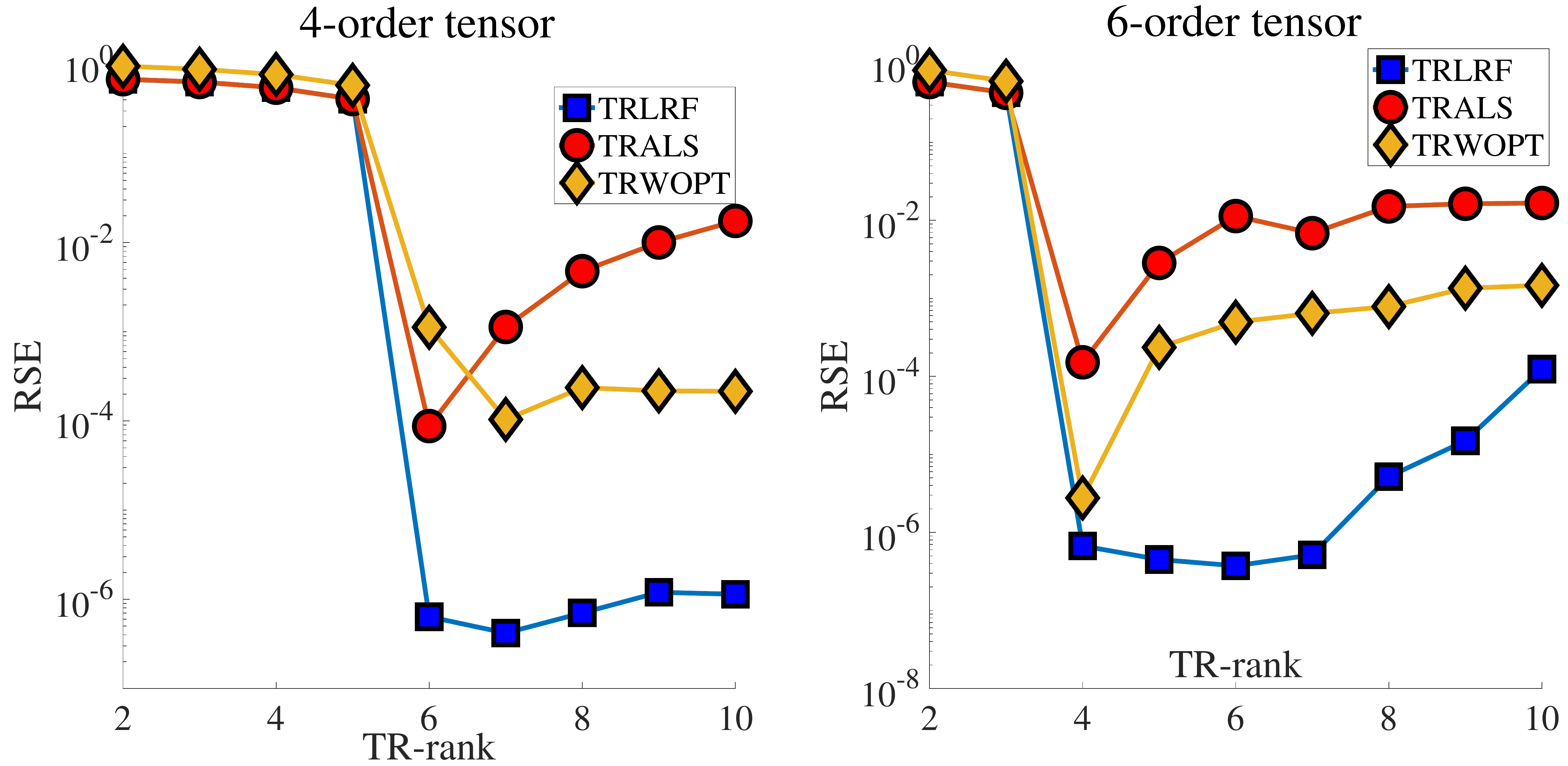}
\end{center}
\caption{Completion performance of three TR-based algorithms in the synthetic data experiment. The RSE values of different selected TR-ranks are recorded. The missing rate of the two target tensors is $0.5$ and the real TR-ranks are 6 and 4 respectively.}
\label{rs_fig}
\end{figure}

Figure \ref{rs_fig} shows the final RSE results which represent the average values of 100 independent experiments for each case. From the figure, we can see that all the three algorithms had their lowest RSE values when the real TR-ranks of the tensors were chosen and the best performance was obtained from our TRLRF. Moreover, when the TR-rank increased, the performance of TRLRF remained stable while the performance of the other two compared algorithms fell drastically. This indicates that imposing low-rankness assumption on the TR factors can bring robustness to rank selection, which largely alleviates the model selection problem in the experiments.

\subsection{Benchmark images inpainting}

\begin{figure}[htb]
\begin{center}
\includegraphics[width=0.7\linewidth]{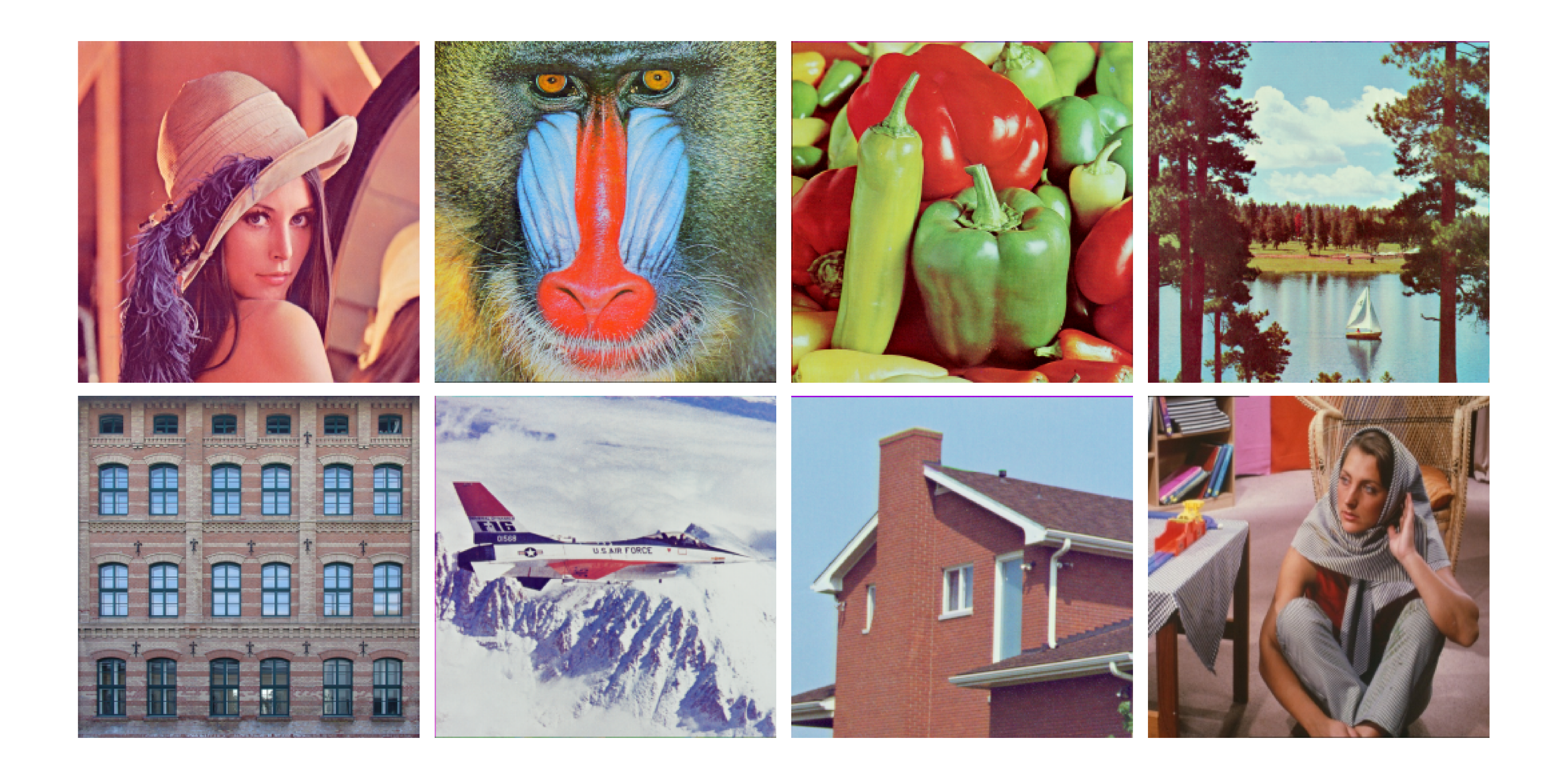}
\end{center}
\caption{The eight benchmark images. The first image is named ``Lena" and is used in the next two experiments.}
\label{bm_fig}
\end{figure}

\begin{figure}[htb]
\begin{center}
\includegraphics[width=0.95\linewidth]{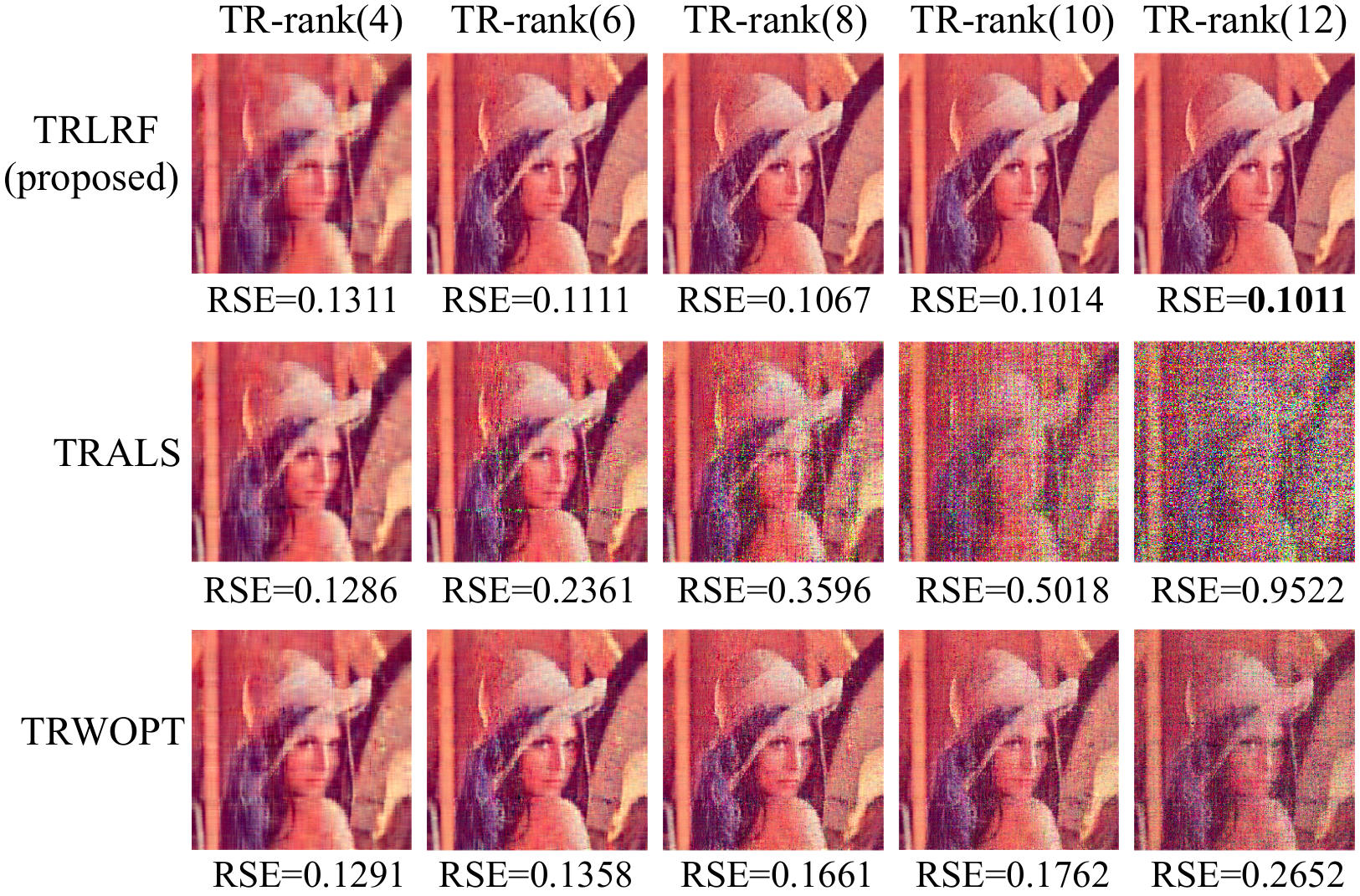}
\end{center}
\caption{Visual completion results of the TRLRF (proposed), TRALS, and TRWOPT on image ``Lena'' with different TR-ranks, when the missing rate is $0.8$. The selected TR-ranks are 4, 6, 8, 10, 12 respectively, from the first column to the last column. The RSE results are noted under each picture.}
\label{img_robust}
\end{figure}

\begin{figure}[htb]
\begin{center}
\includegraphics[width=1\linewidth]{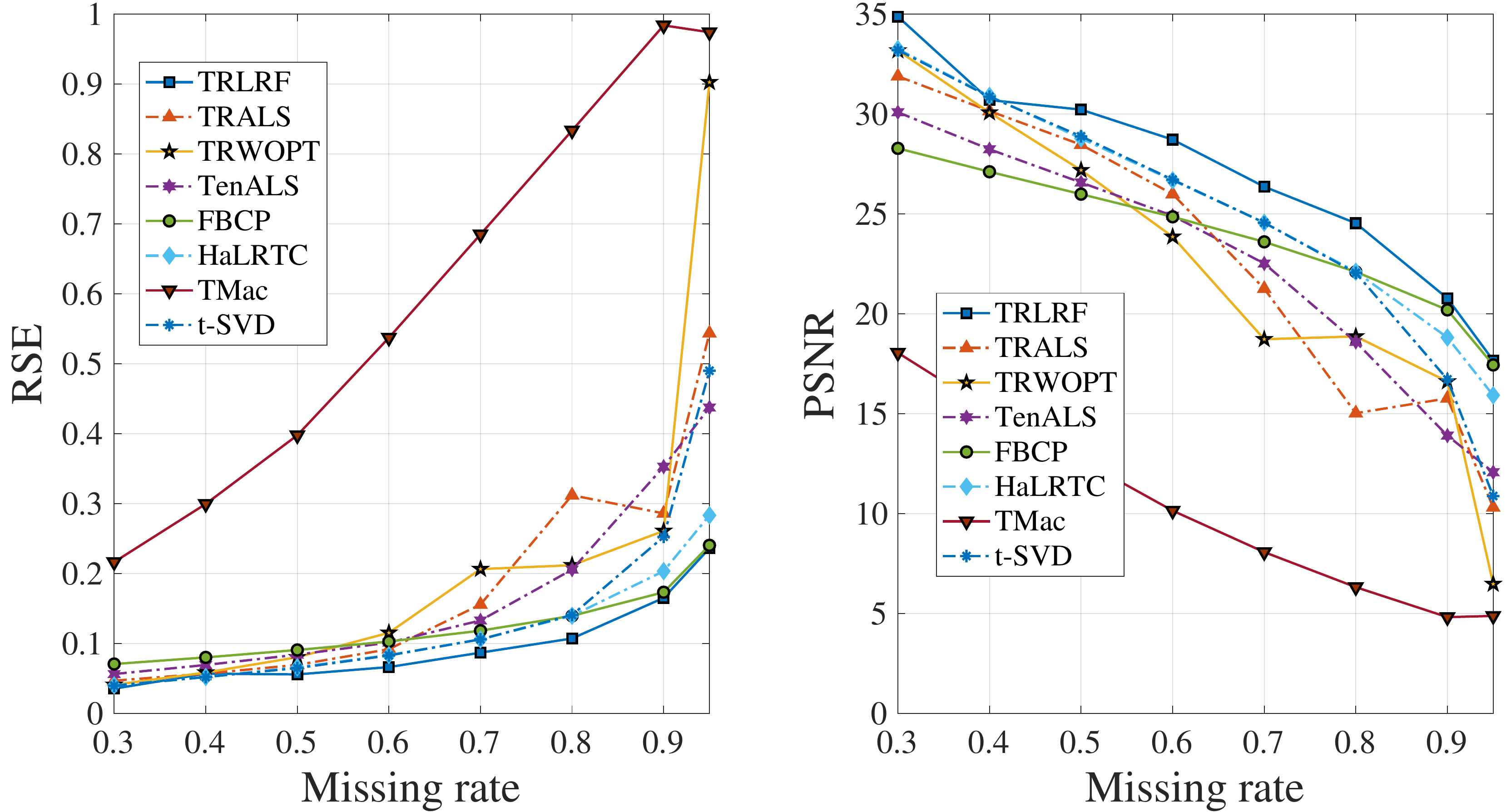}
\end{center}
\caption{Average completion performance of the eight considered algorithms, under different data missing rates.}
\label{img_compare}
\end{figure}

In this section, we tested our TRLRF against the state-of-the-art algorithms on eight benchmark images which are shown in Figure \ref{bm_fig}. The size of each RGB image was $256\times\ 256 \times3$ which can be considered as a three-order tensor. For the first experiment, we continued to verify the TR-rank robustness of TRLRF on the image named ``Lena". Figure \ref{img_robust} shows the completion results of TRLRF, TRALS, and TRWOPT when different TR-ranks for each algorithm are selected. The missing rate of the image was set as $0.8$, which is the case that the TR decompositions are prone to overfitting. From the figure, we can see that our TRLRF gives better results than the other two TR-based algorithms in each case and the highest performance was obtained when the TR rank was set as 12. When TR-rank increases, the completion performance of TRALS and TRLRF decreases due to redundant model complexity and overfitting of the algorithms, while our TRLRF shows better results even the selected TR-rank is larger than the desired TR-rank.

In the next experiment, we compared our TRLRF to the two TR-based algorithm, TRALS and TRWOPT, and the other state-of-the-art algorithms, i.e., TenALS \cite{jain2014provable}, FBCP \cite{zhao2015bayesian}, HaLRTC \cite{liu2013tensor}, TMac \cite{xu2013parallel} and t-SVD \cite{zhang2014novel}. We tested these algorithms on all the eight benchmark images and for different missing rates: $0.3, 0.4, 0.5, 0.6, 0.7, 0.8, 0.9$ and $0.95$. The relative square error (RSE) and peak signal-to-noise ratio (PSNR) were adopted for the evaluation of the completion performance. For RGB image data, PSNR is defined as
$\text{PSNR}=10\log_{10}(255^2/\text{MSE})$ where MSE is calculated by $\text{MSE}=\Vert \tensor{T}_{real}-\tensor{X} \Vert_F^2/\text{num}(\tensor{T}_{real})$, and num($\cdot$) denotes the number of element of the fully observed tensor. 

For the three TR-based algorithms, we assumed the TR-ranks were equal for every core tensor (i.e., $R_1=R_2=\ldots=R_N$). The best completion results for each algorithm were obtained by selecting best TR-ranks for the TR-based algorithms by a cross-validation method. Actually, finding the best TR-rank to obtain the best completion results is very tedious. However, this is much easier for our proposed algorithm because the performance of TRLRF is fairly stable even though the TR-rank is selected from a wide large. For the other five compared algorithms, we tuned the hyper-parameters according to the suggestions of each paper to obtain the best completion results. Finally, we show the average performance of the eight images for each algorithm under different missing rates by line graphs. Figure \ref{img_compare} shows the RSE and PSNR results of each algorithm. The smaller RSE value and the larger PSNR value indicate the better performance. Our TRLRF performed the best among all the considered algorithms in most cases. When the missing rate increased, the completion results of all the algorithms decreased, especially when the missing rate was near $0.9$. The performance of most algorithm fell drastically when the missing rate was $0.95$. However, the performance of TRLRF, HaLRTC, and FBCP remained stable and the best performance was obtained from our TRLRF.

\begin{table*}[h]
\centering
\caption{HSI completion results (RSE) under three different tensor orders with different rank selections}
\label{hsi}
\footnotesize
\begin{tabular}{c|c|c|c|c|c|c|c|c}
\hline
& TRLRF  & TRALS & TRWOPT & TMac & TenALS & t-SVD & FBCP &HaLRTC \\
\hline
$\text{3-order}, \text{high-rank} (R_n=12)$&\textbf{0.06548} &0.07049 &0.06695 &0.1662&0.3448&0.4223 &0.2363&0.1254\\
$\text{3-order}, \text{low-rank} (R_n=8)$ &\textbf{0.1166} &0.1245 &0.1249 &0.2963&0.3312&- &-&-\\
$\text{5-order}, \text{high-rank} (R_n=22)$ &\textbf{0.1035}&0.1392&0.1200 &0.8064&-&0.9504 &0.3833&0.3944\\
$\text{5-order}, \text{low-rank} (R_n=18)$ &\textbf{0.1062}&0.1122&0.1072 &0.7411&-&- &-&-\\
$\text{8-order}, \text{high-rank} (R_n=24)$ &\textbf{0.1190}&0.1319 &0.1637&0.9487&-&0.9443 &0.4021 &0.9099\\
$\text{8-order}, \text{low-rank} (R_n=20)$ &\textbf{0.1421}&0.1581 &0.1767&0.9488&-&0.9450 &0.4135 &0.9097\\
\hline
\end{tabular}
\end{table*}

\begin{figure*}[h]
\begin{center}
\includegraphics[width=0.95\linewidth]{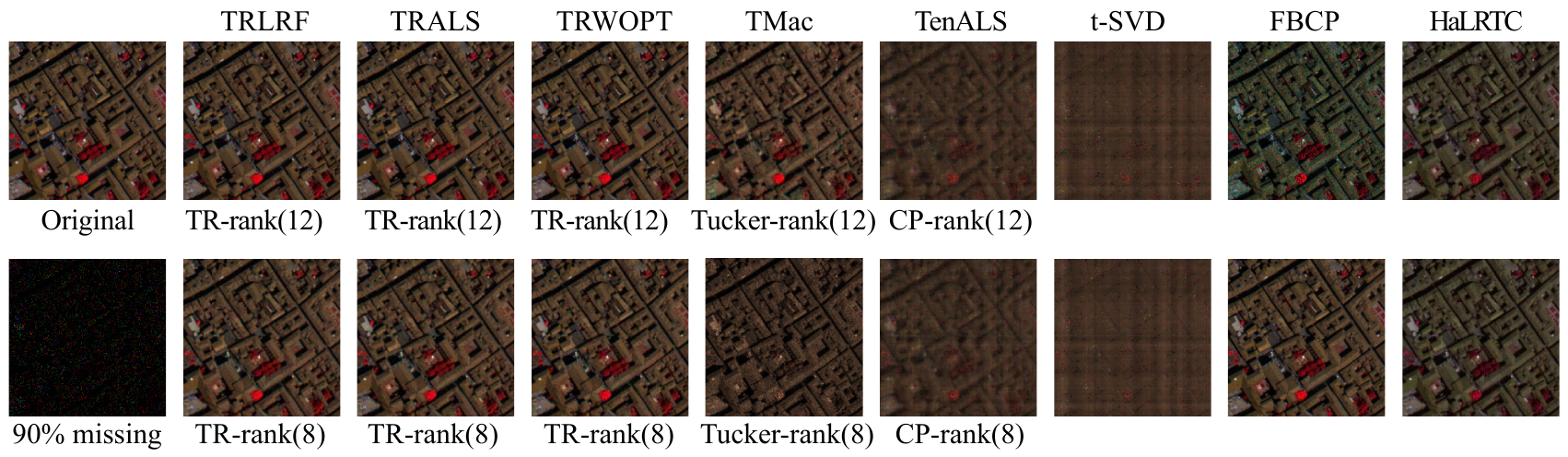}
\end{center}
\caption{Completion results under the $0.9$ missing rate HSI data. The channels 80, 34, 9 are picked to show the visual results. The rank selection of TRLRF, TRALS, TRWOPT, TMac and TenALS are given under the corresponding images.}
\label{hsi_img}
\end{figure*}

\subsection{Hyperspectral image}
A hyperspectral image (HSI) of size $200\times\ 200\times80$ which records an area of the urban landscape was tested in this section\footnote{\url{http://www.ehu.eus/ccwintco/index.php/Hyperspectral_Remote_Sensing_Scenes}}. In order to test the performance of TRLRF on higher-order tensors, the HSI data was reshaped to higher-order tensors, which is an easy way to find more low-rank features of the data. We compared our TRLRF to the other seven tensor completion algorithms in 3-order tensor ($200\times\ 200\times80$), 5-order tensor ($10\times 20\times 10 \times 20 \times 80$) and 8-order tensor ($8\times 5\times 5 \times 8 \times 5\times 5\times  8\times10$) cases. The higher-order tensors were generated from original HSI data by directly reshaping it to the specified size and order.

The experiment aims to verify the completion performance of the eight algorithms under different model selection, whereby the experiment variables are the tensor order and tensor rank. The missing rates of all the cases are set as $0.9$. All the tuning parameters of every algorithm were set according to the statement in the previous experiments. Besides, for the experiments which need to set rank manually, we chose two different tensor ranks: high-rank and low-rank for algorithms. It should be noted that the CP-rank of TenALS and the Tucker-rank of TMac were set to the same values as TR-rank. The completion performance of RSE and visual results are listed in Table \ref{hsi} and shown in Figure \ref{hsi_img}. The results of FBCP, HaLRTC and t-SVD were not affected by tensor rank, so the cases of the same order with different rank are left blank in Table \ref{hsi}. The TenALS could not deal with tensor more than three-order, so the high-order tensor cases for TenALS are also left blank. As shown in Table \ref{hsi}, our TRLRF gives the best recovery performance for the HSI image. In the 3-order cases, the best performance was obtained when the TR-rank was 12, however, when the rank was set to 8, the performance of TRLRF, TRALS, TRWOPT, TMac, and TenALS failed because of the underfitting of the selected models. For 5-order cases, when the rank increased from 18 to 22, the performance of TRLRF kept steady while the performance of TRALS, TRWOPT, and TMac decreased. This is because the high-rank makes the models overfit while our TRLRF performs without any issues, owing to its inherent TR-rank robustness. In the 8-order tensor cases, similar properties can be obtained and our TRLRF also performed the best.

\section{Conclusion}
We have proposed an efficient and high-performance tensor completion algorithm based on TR decomposition, which employed low-rank constraints on the TR latent space. The model has been efficiently solved by the ADMM algorithm and it has been shown to effectively deal with model selection which is a common problem in most traditional tensor completion methods, thus providing much lower computational cost. The extensive experiments on both synthetic and real-world data have demonstrated that our algorithm outperforms the state-of-the-art algorithms. Furthermore, the proposed method is general enough to be extended to various other tensor decompositions in order to develop more efficient and robust algorithms.

\bibliographystyle{aaai}
\bibliography{paper}

\end{document}